\documentclass[letterpaper, 10 pt, journal, twoside]{IEEEtran}
\usepackage{amsmath,amsfonts,amssymb,amsthm}
\usepackage[ruled,vlined,linesnumbered]{algorithm2e}
\usepackage{array}
\usepackage[caption=false,font=normalsize,labelfont=sf,textfont=sf]{subfig}
\usepackage{textcomp}
\usepackage{stfloats}
\usepackage{url}
\usepackage{verbatim}
\usepackage{graphicx}
\usepackage{cite}
\usepackage{xcolor}
\usepackage{booktabs}
\usepackage{multirow}
\usepackage{mathrsfs}
\usepackage{threeparttable}
\theoremstyle{plain}

\newtheorem{theorem}{Theorem}

\theoremstyle{definition}
\newtheorem{problem}{Problem}
\newtheorem{definition}{Definition}

\theoremstyle{remark}

\SetKwComment{tcp}{/* }{ */}

\begin{document}

\title{Homotopy-Aware Corridor Generation without \\ Predefined Reference Paths}

\author{
Haoze Dong,~\IEEEmembership{Graduate Student Member,~IEEE}, 
Minghan Li, 
Meng Guo,~\IEEEmembership{Member,~IEEE}, \\
and Zhongkui Li,~\IEEEmembership{Senior Member,~IEEE}
\thanks{Manuscript received: March 25, 2026 ; Revised June 20, 2026; Accepted July 19, 2026.
This paper was recommended for publication by Editor Ashis Banerjee upon evaluation of the Associate Editor and Reviewers' comments.
This work was supported by the National Natural Science Foundation of China under grants 62425301, U2241214, T2121002.
(\textit{Corresponding author: Zhongkui Li})}
\thanks{The authors are with the School of Advanced Manufacturing and Robotics, Peking University, Beijing 100871, China (e-mail: zhongkli@pku.edu.cn).}
\thanks{Code and Video: https://github.com/HauserDong/path-free-gcs-corridors}
\thanks{Digital Object Identifier (DOI): see top of this page.}
}

\markboth{IEEE ROBOTICS AND AUTOMATION LETTERS. PREPRINT VERSION. ACCEPTED July, 2026}
{Dong \MakeLowercase{\textit{et al.}}: Homotopy-Aware Corridor Generation without Predefined Reference Paths}


\maketitle

\begin{abstract}
Generating safe corridors is essential for collision-free robotic motion planning, yet most existing methods rely on predefined reference paths, which bias corridor geometry and implicitly limit the homotopy classes that can be explored. 
We propose a reference-path-free corridor generation framework on graphs of convex sets (GCS) that constructs corridors directly as sequences of convex sets, allowing corridor structure to emerge from the free-space representation rather than from a guiding path. 
To reason about similarity among corridors, we extend visibility-based deformation from paths to convex-set sequences, enabling the fusion of topologically redundant corridors while preserving distinct alternatives. 
To overcome the limited adaptability of existing GCS methods based on static global decompositions, we further develop an adaptive multi-scale GCS, in which a sampling-based fine-scale graph supports localized updates and a visibility-based coarse-scale graph enables compact global exploration. 
The two levels maintain topological consistency, allowing incremental updates without full graph reconstruction under environmental uncertainty. 
Numerical experiments characterize GCS construction, corridor generation, homotopy-aware exploration, and local updates, showing efficient graph construction, stable trajectory-level performance, and shorter-duration homotopy-aware trajectories than existing baselines.
Hardware experiments on ground and aerial robots, including deployment with onboard localization, further validate the framework under translated and previously unknown obstacles.
\end{abstract}


\section{Introduction}
\IEEEPARstart{I}{n} modern robotic motion planning, trajectory optimization can be performed either directly in configuration space using continuous optimization techniques~\cite{ratliff2009chomp,kalakrishnan2011stomp}, or within a structured intermediate representation that constrains the feasible region~\cite{liu2017planning,chen2023MultiRobot}.
In cluttered environments, this representation often takes the form of a corridor that encodes safe admissible regions for downstream optimization.
Corridors are commonly generated by inflating a discrete reference path, thereby inheriting its geometry and topology.
However, lifting a path-level representation to a region-level corridor can cause sensitivity to local path artifacts and implicit restriction of the explored topology, particularly in environments with multiple homotopy classes.
This raises a fundamental question: can corridors be generated and reasoned about directly, without relying on a discrete reference path as a geometric prior?

\begin{figure}[!t]
  \centering
  \includegraphics[width=1.0\linewidth]{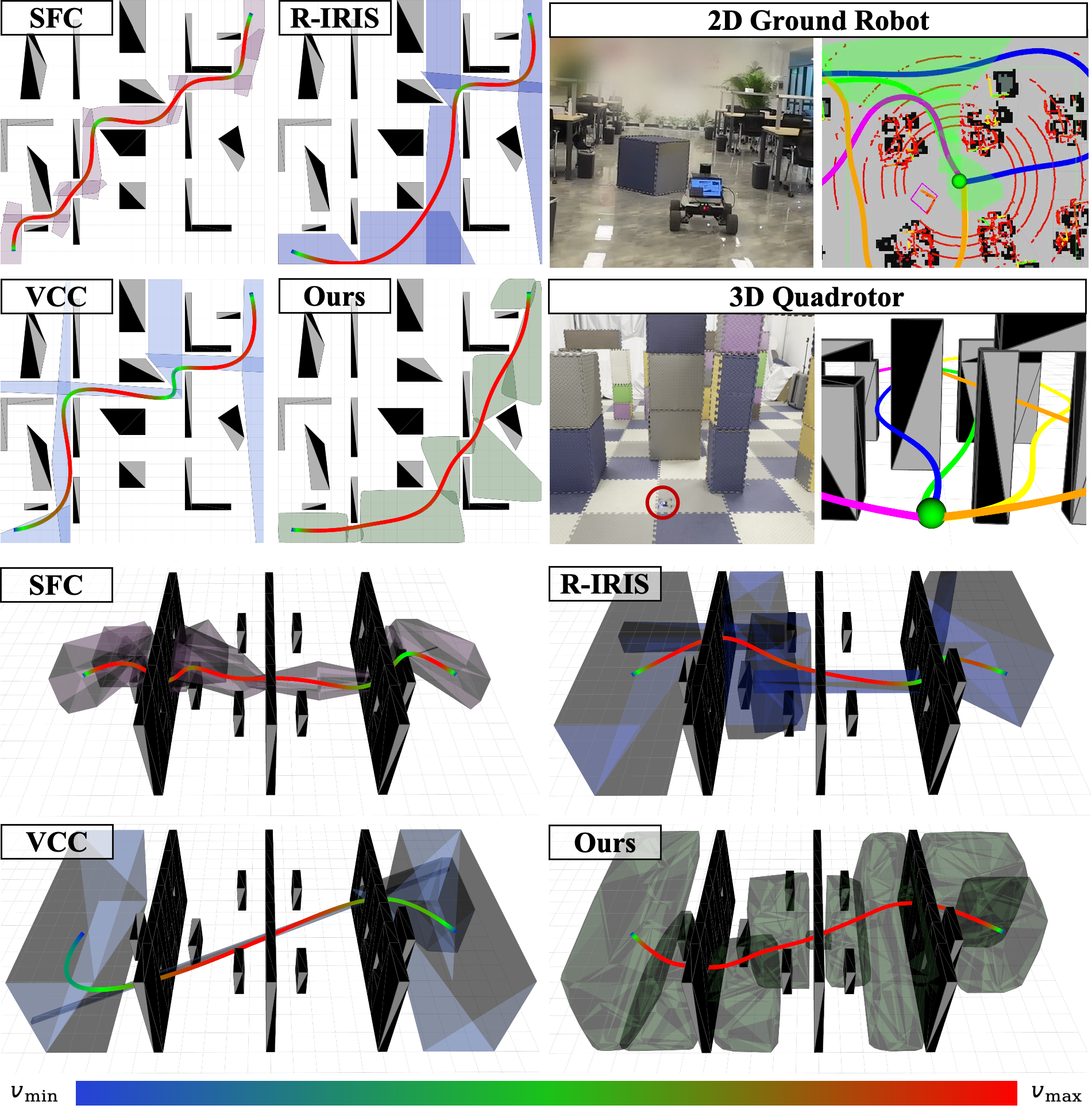}

  \vspace{-0.2cm}

  \caption{
    \textbf{Top Left:} Comparison of trajectories, corridors, and speed profiles in a 2-D multi-topology environment.
    \textbf{Top Right:} Experimental results on a 2-D ground robot and a 3-D quadrotor.
    \textbf{Bottom:} Comparison of trajectories, corridors and speed profiles in a 3-D multi-topology environment.
  }
  \label{fig:all_results}

  \vspace{-0.4cm}
  
\end{figure}

\subsection{Related Work}
\subsubsection{Path-Based Corridor Generation}
A common approach to corridor generation constructs a sequence of convex sets (SCS) by inflating a discrete reference path \cite{liu2017planning,chen2023MultiRobot, park2022online}. 
In this paradigm, corridor geometry and topology are implicitly inherited from the underlying path, which serves as a low-dimensional geometric prior. 
While effective in many scenarios, downstream performance depends on the guiding path and the resulting region quality, which may restrict the spatial freedom available for trajectory optimization (see Fig.~8 in \cite{wang2025fast} and Sec.~VIII.A in \cite{WernerP-RSS-25}). 
Moreover, reliance on a single reference path typically restricts exploration to the homotopy class represented by that path, limiting systematic reasoning about topological diversity at the corridor level.

\subsubsection{Planning with Graphs of Convex Sets}
A GCS represents free space as convex regions connected by nonempty intersections, coupling graph search with convex trajectory optimization over region sequences \cite{marcucci2023motion,Natarajan-RSS-24} and enabling compact modeling of complex environments~\cite{werner2024approximating}.
Once a GCS is available, shortest paths can be formulated and solved using mixed-integer convex optimization and convex relaxations~\cite{marcucci2024shortest}, while GCS relaxations can guide nonconvex trajectory optimization over candidate region sequences~\cite{von2024using}.
However, these methods generally rely on a precomputed global convex decomposition that may contain redundant region sequences, while corridor generation, corridor-level homotopy reasoning, and localized graph maintenance remain largely unaddressed.

\subsubsection{Homotopy-Aware Motion Planning}
Homotopy-aware motion planning has been extensively studied as a means to generate topologically distinct solutions in complex environments. 
Most methods define and reason about homotopy classes at the level of paths, using tools such as H-signature\cite{bhattacharya2012topological}, D-Signature\cite{wang2025customize}, winding numbers\cite{berger2001topological}, or visibility-based criteria\cite{jaillet2008path,zhou2020robust}. 
However, trajectory optimization and feasibility constraints are often imposed at the corridor or region level rather than on paths. 
This mismatch motivates the need for homotopy reasoning at the level of corridors, which is not addressed by existing path-centric formulations.

\subsection{Our Method}

\begin{figure}[!t]
  \centering
  \includegraphics[width=1.0\linewidth]{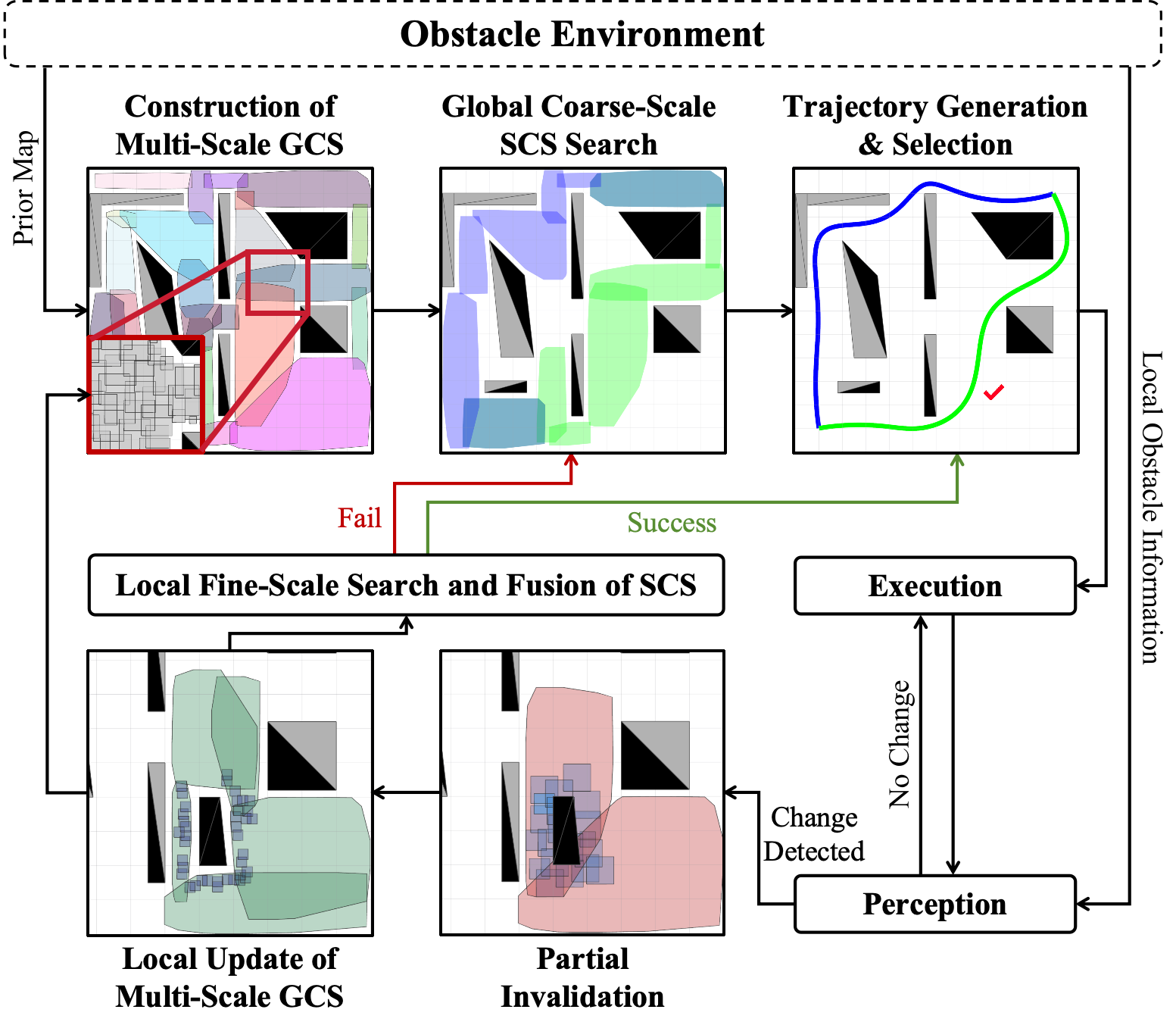}
  
  \vspace{-0.2cm}

  \caption{Overall framework with offline initialization and online adaptation. The offline phase constructs a multi-scale GCS structure, enabling global discovery of topology-distinct corridors. The online phase performs localized updates without requiring full graph reconstruction.}
  \label{fig:framework}

  \vspace{-0.4cm}

\end{figure}

As illustrated in Fig.~\ref{fig:framework}, we propose a homotopy-aware corridor generation framework that constructs corridors directly as SCSs from the free-space representation, without predefined reference paths.
We extend visibility deformation (VD) and uniform visibility deformation (UVD) from paths to SCSs to assess corridor-level similarity, fuse redundant corridors, and retain topology-distinct alternatives.
We develop an adaptive multi-scale GCS comprising a sampling-based fine-scale GCS (F-GCS) and a visibility-based coarse-scale GCS (C-GCS).
The coarse level enables compact discovery of topology-distinct corridors, while the fine level supports localized updates under environmental uncertainty without the need for full graph reconstruction. 

The main contributions of this work are threefold:
(I) A reference-path-free corridor generation formulation is introduced, where corridors are searched directly as SCSs, avoiding the geometric and topological bias induced by reference paths;
(II) A corridor-level homotopy reasoning framework is established by extending VD and UVD from paths to SCSs, enabling similarity assessment and redundancy reduction;
(III) An adaptive multi-scale GCS is developed to overcome the limitation of classical GCS methods based on static global decompositions, enabling efficient discovery of topology-distinct corridors and localized updates under map uncertainty.
\section{Problem Formulation} \label{sec:problem-formulation}

Let $\mathcal{C} \subseteq \mathbb{R}^d$ be the configuration space, and $\mathcal{O} \triangleq \{O_j\}_{j=1}^{M}$ the obstacle set, where each $O_j \subseteq \mathcal{C}$ is convex. Nonconvex obstacles are approximated by convex components. The free space is $\mathcal{F} \triangleq \mathcal{C} \setminus \bigcup_j O_j$, approximated by collision-free convex sets $\{\mathcal{L}_i\}_{i=1}^N$, where $\mathcal{L}_i \subseteq \mathcal{F}$. These sets form a GCS $\mathcal{G} \triangleq (\mathcal{V}, \mathcal{E})$, where vertices $v_i \in \mathcal{V}$ correspond to $\mathcal{L}_i$, and edges $(v_i, v_j) \in \mathcal{E}$ exist if $\mathcal{L}_i \cap \mathcal{L}_j \neq \emptyset$. 
Given $q_0,q_T\in\mathcal{F}$, we define the singleton convex sets $\mathcal{L}_s\triangleq\{q_0\}$ and $\mathcal{L}_t\triangleq\{q_T\}$, corresponding to source and target vertices $v_s$ and $v_t$. They are connected to all vertices whose convex sets intersect $\mathcal{L}_s$ and $\mathcal{L}_t$, respectively.
An SCS is an ordered sequence $\mathcal{L}^m \triangleq (\mathcal{L}_1^m, \ldots, \mathcal{L}_{k_m}^m)$ corresponding to a path from $v_s$ to $v_t$ in $\mathcal{G}$, defining a collision-free corridor within which a trajectory can be optimized using standard trajectory optimization methods. Multiple SCSs may exist, but distinct paths on $\mathcal{G}$ may induce topologically similar corridors, leading to redundant computation. Our goal is to maintain one representative SCS per class based on corridor-level criteria.
We assume a static environment with uncertainty: an initial obstacle set $\mathcal{O}_{\mathrm{prior}}$ is available, while online perception may introduce updates that invalidate parts of $\mathcal{G}$. 

\begin{problem}
Given $\mathcal{O}_{\mathrm{prior}}$, $q_0$, $q_T$, and incremental obstacle updates, construct and maintain topologically distinct, collision-free corridors connecting $q_0$ and $q_T$ through paths from $v_s$ to $v_t$, allowing the robot to generate diverse trajectories and select the best one in a static yet uncertain environment.
\hfill $\blacksquare$
\end{problem}
\section{Preliminaries} \label{sec:preliminaries}
A path is a continuous function $p:[0,1]\rightarrow\mathcal{F}$ with
$p(0)=q_0$ and $p(1)=q_T$. 
Homotopy classifies paths based on continuous deformations with fixed endpoints~\cite{allenhatcher2001Algebraic}, but can be too coarse to distinguish practically different solutions in high-dimensional planning problems~\cite{jaillet2008path,zhou2020robust}.

Two points $x,y\in\mathcal{F}$ are \emph{visible} if the line segment $\overline{xy}$ lies entirely in the free space $\mathcal{F}$.
Visibility deformation refines classical homotopy by imposing visibility constraints in the free space.
Uniform visibility deformation further requires this visibility relation to hold pointwise along two parameterized paths.
The path-level VD and UVD notions used in this work are defined as follows.

\begin{definition}[Visibility Deformation (VD)~\cite{jaillet2008path}]
Two paths $p$ and $p'$ with the same endpoints are visibility deformable if one can be continuously deformed into the other through visibility-preserving straight-line connections; equivalently, they belong to the same VD class.
\hfill $\blacksquare$
\end{definition}

\begin{definition}[Uniform Visibility Deformation (UVD)~\cite{zhou2020robust}]
Two paths $p(s)$ and $p'(s)$ with the same endpoints are uniform visibility deformable if $\forall s\in[0,1]$, the segment $\overline{p(s) \, p'(s)}\subseteq\mathcal{F}$; equivalently, they belong to the same UVD class.
\hfill $\blacksquare$
\end{definition}

VD and UVD are stricter than classical homotopy and provide a visibility-aware classification of paths.
We extend them from paths to SCSs, enabling corridor-level topology-aware reasoning for which classical path homotopy is not directly defined.
\section{Proposed Solution} \label{sec:method}
The proposed solution consists of two coupled components.
First, we construct an adaptive multi-scale GCS in Sec.~\ref{sec:adaptive-multi-scale-graphs-of-convex-sets} to capture global connectivity of the free space while supporting localized updates under environmental uncertainty.
Second, we perform homotopy-aware corridor generation directly on the multi-scale GCS in Sec.~\ref{sec:homotopy-aware-corridor-generation}, representing corridors as SCSs and identifying and fusing topologically redundant corridors through corridor-level homotopy reasoning.
Together, as summarized in Sec.~\ref{sec:overall-framework}, these components enable efficient discovery and maintenance of topology-distinct corridors without predefined reference paths.

\subsection{Construction of Multi-Scale Graphs of Convex Sets} \label{sec:adaptive-multi-scale-graphs-of-convex-sets}

Planning in static yet uncertain environments requires a representation that captures global connectivity, retains local geometric fidelity, and supports localized updates.
Existing GCS methods typically operate on a single convex decomposition~\cite{marcucci2023motion,Natarajan-RSS-24,werner2024approximating}. 
At a single scale, coarse convex sets yield compact graphs but reduce geometric fidelity, whereas fine convex sets retain local detail but increase graph size and downstream computation. 
To balance these properties, we propose an adaptive multi-scale GCS $\mathcal{G}_{\mathrm{MS}}$.

\subsubsection{Multi-Scale GCS Architecture} \label{sec:multi-scale-gcs}

The architecture consists of a fine-scale representation and a coarse-scale abstraction.

\begin{definition}
The fine-scale GCS is defined as
$\mathcal{G}_f \triangleq (\mathcal{V}_f, \mathcal{E}_f)$,
where each $v_i^f \in \mathcal{V}_f$ corresponds to a convex set $\mathcal{L}_i^f \subseteq \mathcal{F}$.
The size of $\mathcal{L}_i^f$ is defined as its maximum coordinate-wise spatial extent, 
$
\operatorname{size}(\mathcal{L}_i^f)
\triangleq
\sup_{x,y\in \mathcal{L}_i^f}\|x-y\|_{\infty}.
$
Each fine-scale convex set satisfies $\operatorname{size}(\mathcal{L}_i^f)\leq \varepsilon$, where $\varepsilon>0$ is the prescribed fine-scale spatial resolution.
An edge $(v_i^f, v_j^f) \in \mathcal{E}_f$ exists if and only if
$\mathcal{L}_i^f \cap \mathcal{L}_j^f \neq \emptyset$.
\hfill $\blacksquare$
\end{definition}

The F-GCS captures local free-space geometry with high fidelity and serves as the underlying substrate for efficient localized updates under environmental uncertainty.

\begin{definition}
The coarse-scale GCS is defined as
$\mathcal{G}_c \triangleq (\mathcal{V}_c, \mathcal{E}_c)$,
where each $v_i^c \in \mathcal{V}_c$ corresponds to a convex set
\begin{equation} \label{eq:coarse-scale-set}
    \mathcal{L}_i^c
    \triangleq
    \operatorname{Conv}
    \left(
        \bigcup_{k \in \mathcal{I}(v_i^c)}
        \mathcal{L}_k^f
    \right) 
    \subseteq \mathcal{F},
\end{equation}
where $\mathrm{Conv}(\cdot)$ is the convex hull and $\mathcal{I}(v_i^c) \subseteq \{1,\dots,|\mathcal{V}_f|\}$ denotes a nonempty index set of supporting fine-scale vertices, selected such that the resulting convex hull is collision-free.
An edge $(v_i^c, v_j^c) \in \mathcal{E}_c$ exists if and only if
$\mathcal{L}_i^c \cap \mathcal{L}_j^c \neq \emptyset$.
\hfill $\blacksquare$
\end{definition}

The C-GCS provides a compact abstraction of global connectivity, enabling efficient exploration of corridor candidates.

\begin{problem}
Given a prior map $\mathcal{O}_{\mathrm{prior}}$, resolution $\varepsilon$, and optionally a current graph $\mathcal{G}_{\mathrm{MS}}$ and a local obstacle update $\Delta\mathcal{O}$, construct $\mathcal{G}_{\mathrm{MS}}$ offline or update it locally.
\hfill $\blacksquare$
\end{problem}

As summarized in Alg.~\ref{alg:maintain-msgcs}, the proposed multi-scale GCS is maintained by a unified procedure for offline construction and online local updates.
Offline, the F-GCS is built by sampling collision-free configurations and expanding them into convex sets of size at most $\varepsilon$; in our implementation, these sets are axis-aligned hypercubes, so $\varepsilon$ is the nominal side length of each square in 2-D and cube in 3-D.
Obstacles are inflated by the robot footprint and, when needed, an additional sensing or localization uncertainty margin.
The resolution $\varepsilon$ is chosen below the characteristic width of task-relevant narrow passages.
Since smaller $\varepsilon$ improves local fidelity but increases the number of F-GCS vertices and the construction cost, especially in 3-D, we use it as a nominal upper bound and reduce the convex-set size adaptively only in locally uncovered or narrow free-space regions.
The C-GCS is then formed by selecting mutually invisible roots and performing BFS-based aggregation on $\mathcal{G}_f$, while covered fine-scale sets are excluded and uncovered ones are selected as new roots.

\begin{algorithm}[t]
\caption{\textsc{MaintainMSGCS}}
\label{alg:maintain-msgcs}
\KwIn{Prior map $\mathcal{O}_{\mathrm{prior}}$, resolution $\varepsilon$, optional current graph $\mathcal{G}_{\mathrm{MS}}$ together with obstacle update $\Delta\mathcal{O}$}
\KwOut{Updated multi-scale GCS $\mathcal{G}_{\mathrm{MS}}$}

$\mathcal{R} \leftarrow \mathcal{F}$ if $\mathcal{G}_{\mathrm{MS}}=\emptyset$, else the affected spatial region induced by $\Delta\mathcal{O}$ \;
Remove invalid fine-scale sets in $\mathcal{R}$ and affected coarse-scale sets, together with their incident edges \;
$\{\mathcal{L}_i^{f}\}_{\mathcal{R}} \leftarrow$ Sample in $\mathcal{R}$ and expand into collision-free convex sets of size $\leq \varepsilon$ \;
Update fine-scale adjacency in $\mathcal{R}$ and obtain $\mathcal{G}_f$ \;

$\{v_r^f\}_{\mathcal{R}} \leftarrow$ Select mutually invisible roots from the updated fine-scale subgraph in $\mathcal{R}$ \;
\ForEach{root node in $\{v_r^f\}_{\mathcal{R}}$}{
    Perform BFS on $\mathcal{G}_f$ restricted to $\mathcal{R}$, adding supporting vertices while the convex hull in \eqref{eq:coarse-scale-set} remains collision-free \;
    Construct the corresponding coarse-scale set $\mathcal{L}_i^c$ \;
}
Update coarse-scale adjacency for affected coarse-scale sets and obtain $\mathcal{G}_c$ \;

\KwRet{$\mathcal{G}_{\mathrm{MS}} = (\mathcal{G}_f, \mathcal{G}_c)$}
\end{algorithm}

\subsubsection{Online Local Update and Adaptivity} \label{sec:online-adaptivity}

During online execution, Alg.~\ref{alg:maintain-msgcs} is invoked in its local-update mode.
The spatial change induced by $\Delta\mathcal{O}$ is decomposed into the newly occupied region $\Delta\mathcal{O}^{+}$ and the newly freed region $\Delta\mathcal{O}^{-}$, both contained in the affected region $\mathcal{R}$.
The algorithm restricts reconstruction to $\mathcal{R}$, updates the fine-scale subgraph locally, and propagates the resulting modifications to the coarse scale.
We first identify the fine-scale vertices invalidated by $\Delta\mathcal{O}^{+}$:
\begin{equation}
\tilde{\mathcal{V}}_f
\triangleq
\left\{
v_k^f \in \mathcal{V}_f
\mid
\mathcal{L}_k^f \cap \Delta\mathcal{O}^{+} \neq \emptyset
\right\}.
\end{equation}
Vertices in $\tilde{\mathcal{V}}_f$ are removed from $\mathcal{G}_f$ together with their incident edges in $\mathcal{E}_f$.
Local resampling within $\mathcal{R}$ then restores coverage around removed unsafe sets and covers newly freed space in $\Delta\mathcal{O}^{-}$.
Coarse-scale updates are induced by these fine-scale modifications.
Specifically, an existing coarse-scale vertex $v_i^c \in \mathcal{V}_c$ is invalidated if it is supported by any removed fine-scale vertex, i.e.,
$
\{v_k^f \mid k \in \mathcal{I}(v_i^c)\}
\cap
\tilde{\mathcal{V}}_f
\neq
\emptyset.
$
Such coarse-scale vertices are reconstructed from the updated local fine-scale subgraph, while newly sampled fine-scale sets may form additional coarse-scale vertices.
Thus, updates are confined to the affected subgraphs of $\mathcal{G}_f$ and $\mathcal{G}_c$.
This enables localized adaptation to map updates while retaining a stable global abstraction for corridor generation.

\begin{figure}[!t]
  \centering
  \includegraphics[width=0.9\linewidth]{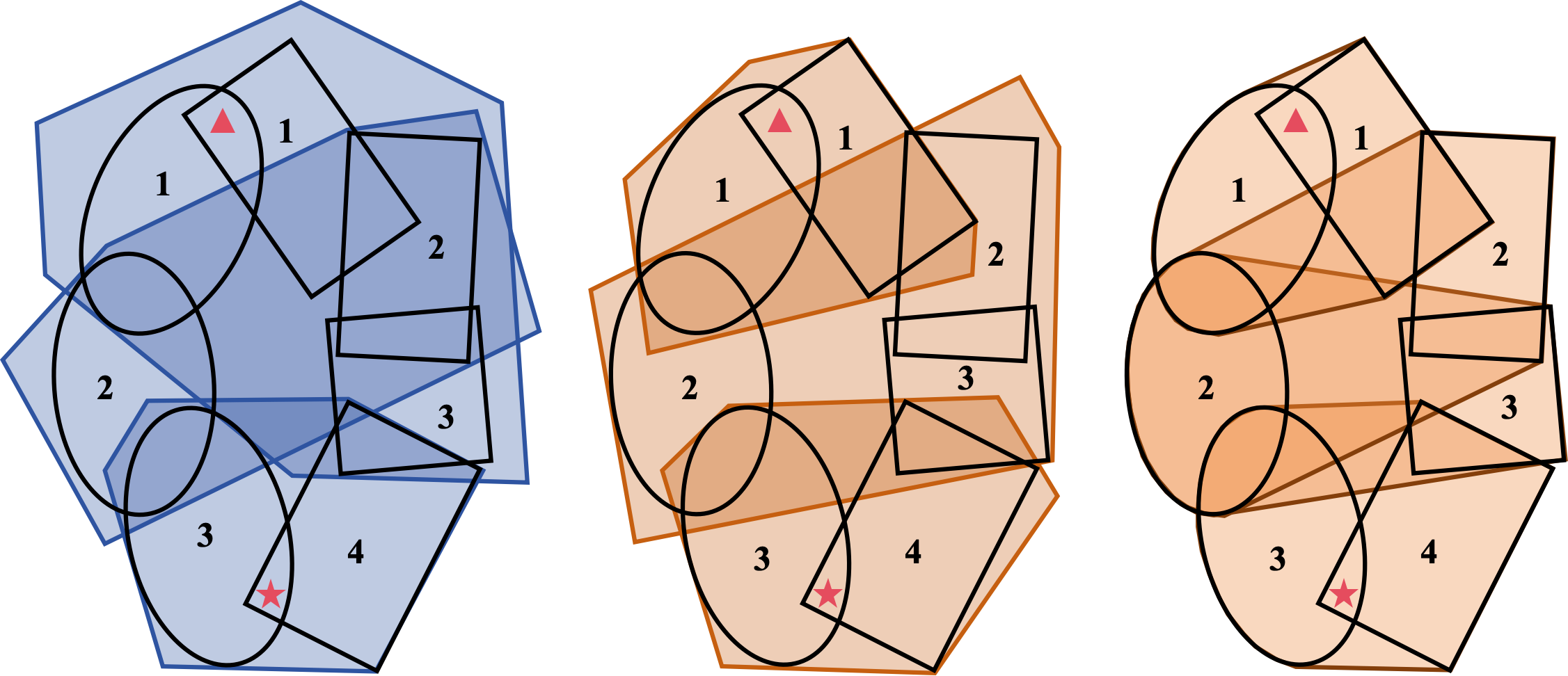}

  \vspace{-0.2cm}

  \caption{
\textbf{Illustration of VD, UVD, and the UVD criterion.}
The ellipse-shaped and polygonal black outlines denote two indexed input SCSs, and the colored filled regions denote the fused SCS.
Red markers denote the start and goal.
\textbf{Left:} VD requires joint collision-free coverage, without index-order constraints.
\textbf{Center:} UVD further requires a monotone embedding of both original SCSs into the fused SCS.
\textbf{Right:} Theorem~\ref{thm:uvd-scs} constructs the fused SCS through pairwise monotone fusion.
}
  \label{fig:VD-illustration}

  \vspace{-0.4cm}

\end{figure}

\subsubsection{Topological Consistency across Scales}
\label{sec:topological-consistency}

A key property of the multi-scale GCS maintained by Alg.~\ref{alg:maintain-msgcs} is that the abstraction preserves corridor-level topological consistency across scales.
This is analyzed using corridor-level VD.

\begin{definition}
\label{def:vd-between-scs}
Two SCSs $\mathcal{L}^{(1)} = (\mathcal{L}^{(1)}_1, \ldots, \mathcal{L}^{(1)}_m)$ and $\mathcal{L}^{(2)} = (\mathcal{L}^{(2)}_1, \ldots, \mathcal{L}^{(2)}_n)$ with the same singleton source and target sets are visibility deformable, denoted $\mathcal{L}^{(1)} \sim_{\mathrm{VD}} \mathcal{L}^{(2)}$, if there exists a collision-free SCS $\mathcal{L}^{*} = (\mathcal{L}^{*}_1, \ldots, \mathcal{L}^{*}_p)$ such that
$\forall l \in \{1,2\},\; \forall \mathcal{L}^{(l)}_i \in \mathcal{L}^{(l)},\;
\exists \mathcal{L}^{*}_k \in \mathcal{L}^{*}
\;\text{s.t.}\;
\mathcal{L}^{(l)}_i \subseteq \mathcal{L}^{*}_k;\;
\forall \mathcal{L}^{*}_k \in \mathcal{L}^{*},\;
\exists \mathcal{L}^{(1)}_i \in \mathcal{L}^{(1)},\;
\exists \mathcal{L}^{(2)}_j \in \mathcal{L}^{(2)}
\;\text{s.t.}\;
\mathcal{L}^{(1)}_i,\, \mathcal{L}^{(2)}_j \subseteq \mathcal{L}^{*}_k.$
\hfill $\blacksquare$
\end{definition}

As shown in Fig.~\ref{fig:VD-illustration} (left), this definition provides a practical method for assessing topological similarity between corridors. Two SCSs are VD-consistent if they admit a common collision-free covering SCS, without requiring index-order alignment, analogous to path-level VD.

\begin{theorem}
\label{thm:cross-scale-topology}
The multi-scale GCS maintained by Alg.~\ref{alg:maintain-msgcs} is corridor-level VD-consistent across scales.
For any fine-scale SCS $\mathcal{L}^f$, there exists an induced coarse-scale SCS $\mathcal{L}^c$ such that
$\mathcal{L}^f \sim_{\mathrm{VD}} \mathcal{L}^c$.
Any two fine-scale SCSs that induce the same coarse-scale SCS are mutually VD.
\end{theorem}

\noindent\emph{Proof sketch.}
By Alg.~\ref{alg:maintain-msgcs}, coarse-scale sets are constructed by aggregating supporting fine-scale sets and taking their convex hulls as in \eqref{eq:coarse-scale-set}. 
Hence each fine-scale set is contained in some coarse-scale set.
Mapping consecutive fine-scale sets to containing coarse-scale sets and removing consecutive duplicates yields a valid coarse-scale SCS, since intersecting fine-scale sets imply that their containing coarse-scale sets also intersect.
Thus, any fine-scale SCS can be mapped to a coarse-scale SCS that covers it, which implies
$\mathcal{L}^f \sim_{\mathrm{VD}} \mathcal{L}^c$ by Def.~\ref{def:vd-between-scs}.
Moreover, if two fine-scale SCSs induce the same coarse-scale SCS, that coarse sequence is a common collision-free covering SCS for both, so they are mutually VD.
\hfill $\blacksquare$

\subsection{Generation of Homotopy-Aware Corridors} \label{sec:homotopy-aware-corridor-generation}

Based on the multi-scale GCS developed in Sec.~\ref{sec:adaptive-multi-scale-graphs-of-convex-sets}, this subsection generates and maintains topology-distinct corridors under uncertainty by reasoning directly over SCSs.

\begin{problem}
Given $\mathcal{G}_{\mathrm{MS}} = (\mathcal{G}_f, \mathcal{G}_c)$, construct collision-free, topology-distinct SCSs connecting $\mathcal{L}_s$ and $\mathcal{L}_t$.
\hfill $\blacksquare$
\end{problem}

\subsubsection{Corridor Exploration}
\label{sec:corridor-discovery}

Corridors are generated by directly searching for SCSs on the C-GCS without predefined reference paths.
The C-GCS provides a compact abstraction of global free-space connectivity, reducing the search space.
Specifically, we search for
$
\mathcal{L}^c\in\operatorname{Paths}(\mathcal{G}_c,v_s,v_t),
$
where $\operatorname{Paths}(\mathcal{G}_c,v_s,v_t)$ denotes the SCSs induced by $v_s$-to-$v_t$ paths in the augmented C-GCS.
Because the C-GCS is constructed from mutually invisible roots and collision-free convex aggregation, search on the C-GCS promotes topological diversity and reduces redundancy among candidate corridors.
The compact C-GCS therefore enables efficient global corridor exploration.
However, when a local map update invalidates the current corridor, rapid fine-scale re-evaluation is necessary.

\subsubsection{Homotopy Reasoning and Corridor Fusion}
\label{sec:homotopy-fusion}

When environmental updates occur, the F-GCS is repaired before the C-GCS is reconstructed, making it natural to re-evaluate corridors on $\mathcal{G}_f$ and attempt to preserve the topology of the invalidated corridor.
While coarse-scale search provides diverse corridors, fine-scale replanning may generate multiple geometrically different yet topologically redundant corridors. 
To remove such redundancy while preserving corridor topology, we perform corridor-level UVD reasoning.

\begin{definition}
\label{def:uvd-between-scs}
Two SCSs with the same singleton source and target sets are uniform visibility deformable, denoted $\mathcal{L}^{(1)} \sim_{\mathrm{UVD}} \mathcal{L}^{(2)}$, if and only if there exists a common collision-free covering SCS into which the convex sets of each original SCS are embedded monotonically, with consistently aligned intersections between consecutive convex sets.
\hfill $\blacksquare$
\end{definition}

Intuitively, Def.~\ref{def:uvd-between-scs} imposes a corridor-level analogue of uniformity by requiring a convex-set covering that preserves both the ordering of the embedded sets and the alignment of consecutive intersections, as illustrated in Fig.~\ref{fig:VD-illustration} (center).
Given the locally searched fine-scale SCSs $\{\mathcal{L}^{f,r}\}$, we incrementally construct a fused set $\mathscr{L}^{\ast}$. Each candidate SCS is compared against the current representatives in $\mathscr{L}^{\ast}$ using the pairwise UVD criterion in Theorem~\ref{thm:uvd-scs}. If the criterion yields a common fused SCS $\hat{\mathcal{L}}$ for the candidate and an existing representative, the representative is replaced by $\hat{\mathcal{L}}$; otherwise, the candidate is inserted as a new representative. This greedy fusion removes topologically redundant corridors while retaining one representative for each remaining fused group under the traversal order. The following theorem provides the pairwise criterion for this fusion step and constructively yields the common fused SCS, as illustrated in Fig.~\ref{fig:VD-illustration} (right).

\begin{theorem}
\label{thm:uvd-scs}
Let 
$
\mathcal{L}^{(l)} = 
(\mathcal{L}^{(l)}_1,\dots,\mathcal{L}^{(l)}_{m_l}), 
\ l\in\{1,2\}.
$
Then 
$
\mathcal{L}^{(1)} \sim_{\mathrm{UVD}} \mathcal{L}^{(2)}
$
if and only if there exist a collision-free SCS
$
\mathcal{L}^\ast = (\mathcal{L}^\ast_1,\dots,\mathcal{L}^\ast_p)
$
and an index-coupling sequence
$
\{(i_k, j_k)\}_{k=1}^p \subseteq \{1, \ldots, m_1\} \times \{1, \ldots, m_2\}
$
such that:
\begin{enumerate}
    \item[(i)] Boundary-complete monotone coupling:\\
    $ (i_1,j_1)=(1,1),\, (i_p,j_p)=(m_1,m_2),$
    and \\
    $ (i_{k+1}-i_k,j_{k+1}-j_k) \in \{(1,0),(0,1),(1,1)\}. $
    
    \item[(ii)] Pairwise covering:
    $\mathcal{L}^{(1)}_{i_k} \cup \mathcal{L}^{(2)}_{j_k} \subseteq\mathcal{L}_k^\ast.$
    
    \item[(iii)] Consistent intersection alignment:\\
    $\mathcal L^\ast_k \cap \mathcal L^\ast_{k+1}
    \;\supseteq\;
    \big(
    \mathcal L^{(1)}_{i_k} \cap \mathcal L^{(1)}_{i_{k+1}}
    \big)
    \;\cup\;
    \big(
    \mathcal L^{(2)}_{j_k} \cap \mathcal L^{(2)}_{j_{k+1}}
    \big)$.
\end{enumerate}
\end{theorem}

\noindent\emph{Proof sketch.}
If a collision-free SCS satisfying conditions (i)--(iii) exists, it provides a jointly covering SCS with monotone embeddings and aligned intersections, and the two SCSs are UVD by Def.~\ref{def:uvd-between-scs}.
Conversely, if two SCSs are UVD, a jointly covering collision-free SCS exists by definition. 
By segmenting this witness according to changes in the indices of the two original SCSs, one can obtain an index-coupling sequence satisfying (i).
A realization follows by forming convex hulls of coupled convex-set pairs 
$ \mathcal L^\ast_k \triangleq \operatorname{Conv}(\mathcal{L}^{(1)}_{i_k} \cup \mathcal{L}^{(2)}_{j_k})$.
Because each coupled pair is contained in a common collision-free convex witness set, its convex hull remains collision-free and contains both coupled sets, satisfying (ii).
Condition (iii) follows because each original consecutive intersection is contained in both adjacent fused sets.
Therefore, the criterion is both necessary and sufficient.
\hfill $\blacksquare$

\subsection{Planning Framework and Complexity Analysis} \label{sec:overall-framework}

\subsubsection{Planning Pipeline}
Fig.~\ref{fig:framework} and Alg.~\ref{alg:overall-framework} summarize the proposed overall planning framework. 
In the offline stage, the multi-scale GCS $\mathcal{G}_{\mathrm{MS}} = (\mathcal{G}_f, \mathcal{G}_c)$ is built from the prior map $\mathcal{O}_{\mathrm{prior}}$. 
$\mathcal{G}_{\mathrm{MS}}$ is augmented with $v_s$ and $v_t$ according to Sec.~\ref{sec:problem-formulation}, and global corridor discovery is performed on $\mathcal{G}_c$ to identify candidate corridors $\{\mathcal{L}^{c,r}\}$. 
The downstream MINCO trajectory optimizer~\cite{WANG2022GCOPTER} is then used to generate multiple candidate trajectories from $\{\mathcal{L}^{c,r}\}$, and one is selected based on a user-defined cost. 
During online execution, local obstacle updates trigger hierarchical updates to both $\mathcal{G}_f$ and $\mathcal{G}_c$. 
If the current trajectory remains valid under the updated map, execution continues without replanning.
If the current trajectory becomes invalid, fine-scale SCSs $\{\mathcal{L}^{f,r}\}$ are obtained by local search on the updated $\mathcal{G}_f$.
If feasible local candidates are found, they are fused into $\mathscr{L}^{\ast}$ using the UVD-based criterion; otherwise, global replanning is triggered on the updated $\mathcal{G}_c$. 
A new trajectory is then generated from the resulting corridor set.
This hierarchical strategy prioritizes fine-scale local corridor repair and invokes coarse-scale global search only when no feasible local corridor is found.

\subsubsection{Computational Complexity}
Let $N_f$ and $N_c$ be the numbers of fine- and coarse-scale convex sets.
Pairwise F-GCS adjacency construction costs $O(N_f^2)$, while root selection during C-GCS construction costs at most $O(N_fN_cC_{\mathrm{vis}})$, where $C_{\mathrm{vis}}$ is the cost of one visibility test.
Each coarse set examines at most $k_f\ll N_f$ fine nodes. 
Let $C_{\mathrm{col}}(N_v)$ and $H_d(N_v)$ denote the costs of one collision check and one $d$-dim convex-hull construction, respectively.
Thus, coarse-set generation costs
$
O\!\left(N_c[k_fC_{\mathrm{col}}(N_v)+H_d(N_v)]\right),
$
and pairwise C-GCS adjacency construction costs
$
O(N_c^2).
$
Online updates replace the global counts with local counts in $\mathcal R$; under bounded sensing range, local set density, and spatially restricted adjacency checks, their cost is independent of the global graph size.

\begin{algorithm}[t]
\caption{Overall Planning Framework}
\label{alg:overall-framework}

\KwIn{Prior map $\mathcal{O}_{\mathrm{prior}}$, resolution $\varepsilon$, start $q_0$, goal $q_T$}
\KwOut{Executable trajectory $\xi(t)$}

$\mathcal{G}_\mathrm{MS} = (\mathcal{G}_f, \mathcal{G}_c) \leftarrow$ \textsc{MaintainMSGCS}($\mathcal{O}_{\mathrm{prior}}, \varepsilon$) \label{alg-line:build-multi-scale-gcs}\;
$\{\mathcal{L}^{c,r}\} \leftarrow$ BFS on $\mathcal{G}_c$ from $v_s$ to $v_t$\label{alg-line:search-initial-corridors}\;
$\xi(t) \leftarrow$ Generate trajectory from $\{\mathcal{L}^{c,r}\}$\label{alg-line:initial-trajectory}\;

\While{not finished}{
    $\Delta\mathcal{O} \leftarrow$ Update local obstacle information\label{alg-line:sense-environment}\;
    $\mathcal{G}_\mathrm{MS} \leftarrow$ \textsc{MaintainMSGCS}($\mathcal{O}_{\mathrm{prior}}, \varepsilon, \mathcal{G}_\mathrm{MS}, \Delta\mathcal{O}$)\label{alg-line:hierarchical-update}\;

    \If{$\xi(t)$ is invalid under the updated map}{
        $\{\mathcal{L}^{f,r}\} \leftarrow$ Local search on the updated $\mathcal{G}_f$\label{alg-line:search-fine-corridors}\;

        \eIf{$\{\mathcal{L}^{f,r}\} \neq \emptyset$}{
            $\mathscr{L}^{\ast} \leftarrow$ UVD-based fusion of $\{\mathcal{L}^{f,r}\}$\;}
        {
            $\mathscr{L}^{\ast} \leftarrow$ Global search on the updated $\mathcal{G}_c$\;
        }

        $\xi(t) \leftarrow$ Generate trajectory from $\mathscr{L}^{\ast}$\label{alg-line:traj-replan}\;
    }

    Execute the next segment of $\xi(t)$\label{alg-line:execute-next}\;
}

\end{algorithm}
\section{Experimental Validation}

\subsection{Experimental Setup}

All numerical experiments were run on an Intel Core i7-13700K (3.4\,GHz) desktop with 32\,GB RAM.
The host OS was Ubuntu 20.04; Drake-based methods~\cite{drake} ran in an Ubuntu 24.04 Docker container.
C++ implementations used C++17, and Python baselines used Python~3.
Convex-set representations and convex-hull operations used Drake's geometry optimization module, and collision checking used the GJK algorithm~\cite{gilbert2002fast}.
For fairness, all trajectory-level comparisons used the same MINCO trajectory optimizer~\cite{WANG2022GCOPTER}.
For trajectory-level evaluation, each optimized MINCO trajectory was sampled at $\Delta t=0.02\,\mathrm{s}$ to obtain discrete positions $\{\mathbf{p}_i\}_{i=0}^{N}$. We define geometric smoothness as the accumulated turning-angle cost
$
J_{\mathrm{ang}}=\sum_{i=1}^{N-1}\theta_i^2,\quad
\theta_i =
\arccos\left(
\frac{\mathbf{e}_i^\top \mathbf{e}_{i+1}}
{\|\mathbf{e}_i\|_2\|\mathbf{e}_{i+1}\|_2}
\right),
$
where  $\mathbf{e}_i=\mathbf{p}_i-\mathbf{p}_{i-1}$. The metric is reported in $\mathrm{rad}^2$, with lower values indicating smoother directional changes.
A video accompanies this paper as multimedia material.

\subsection{Numerical Experiments}

\subsubsection{GCS Construction Comparison}

\begin{table*}[!t]
\centering
\caption{
GCS Construction Comparison Across Different Environments (Averaged Over 10 Maps)
}
\vspace{-0.2cm}
\label{tab:gcs_construction}
\small
\setlength{\tabcolsep}{4pt}
\begin{tabular}{ll|c|ccc|ccc|ccc|ccc}
\toprule
\multicolumn{2}{c|}{\textbf{Scenario}}
& \multicolumn{1}{c|}{\textbf{Coverage}} 
& \multicolumn{3}{c|}{\textbf{Runtime (s)}} 
& \multicolumn{3}{c|}{\textbf{Average Degree}} 
& \multicolumn{3}{c|}{\textbf{Degeneracy}} 
& \multicolumn{3}{c}{\textbf{Regions}} \\

\multicolumn{2}{c|}{} 
& {} 
& Ours & VCC & R-IRIS 
& Ours & VCC & R-IRIS 
& Ours & VCC & R-IRIS
& Ours & VCC & R-IRIS \\
\midrule

\multirow{3}{*}{2-D Clutter}
& Sparse  & 0.95  & \textbf{0.091} & 0.840 & 0.149 & \textbf{2.889} & 8.031 & 4.961 & \textbf{2.0} & 6.8 & 3.9 & 19.6 & 19.4 & \textbf{13.8} \\
& Medium  & 0.90  & \textbf{0.157} & 1.249 & 0.807 & \textbf{2.730} & 6.958 & 4.147 & \textbf{2.2} & 8.0 & 4.1 & 44.0 & 36.4 & \textbf{28.3}  \\
& Dense   & 0.85   & \textbf{0.276} & 1.368 & 1.960 & \textbf{2.392} & 4.944 & 3.933 & \textbf{2.2} & 6.9 & 3.9 & 53.8 & 42.9 & \textbf{41.8} \\
\midrule

\multirow{3}{*}{3-D Clutter}
& Sparse  & 0.90 & 0.443 & 5.942 & \textbf{0.291} & \textbf{7.032} & 49.644 & 8.533 & \textbf{5.0} & 44.1 & 6.3 & 32.5 & 55.3 & \textbf{15.3} \\
& Medium  & 0.90  & \textbf{1.927} & 13.555 & 2.467 & \textbf{9.244} & 102.055 & 13.138 & \textbf{6.4} & 78.4 & 9.0 & 95.5 & 166.1 & \textbf{46.7}  \\
& Dense   & 0.85  & \textbf{2.478} & 13.275 & 5.811 & \textbf{9.051} & 74.755 & 12.748 & \textbf{6.4} & 59.0 & 8.9 & 133.9 & 178.9 & \textbf{68.2} \\
\midrule

\multirow{3}{*}{2-D Maze}
& Wide    & 0.95  & \textbf{0.268} & 1.688 & 2.728 & \textbf{1.951} & 3.747 & 2.572 & \textbf{1.5} & 4.4 & 2.7 & \textbf{17.1} & 38.1 & 32.7 \\
& Medium  & 0.95 & \textbf{0.663} & 2.492 & 8.650 & \textbf{2.035} & 3.715 & 2.575 & \textbf{1.8} & 5.5 & 3.1 & \textbf{32.0} & 63.5 & 62.7  \\
& Narrow  & 0.90 & \textbf{0.963} & 4.898 & 23.859 & \textbf{1.996} & 2.611 & 2.294 & \textbf{1.6} & 4.4 & 2.9 & \textbf{68.6} & 99.7 & 95.7  \\
\bottomrule
\end{tabular}
\end{table*}

\begin{figure}[!t]
  \centering
  \includegraphics[width=0.90\linewidth]{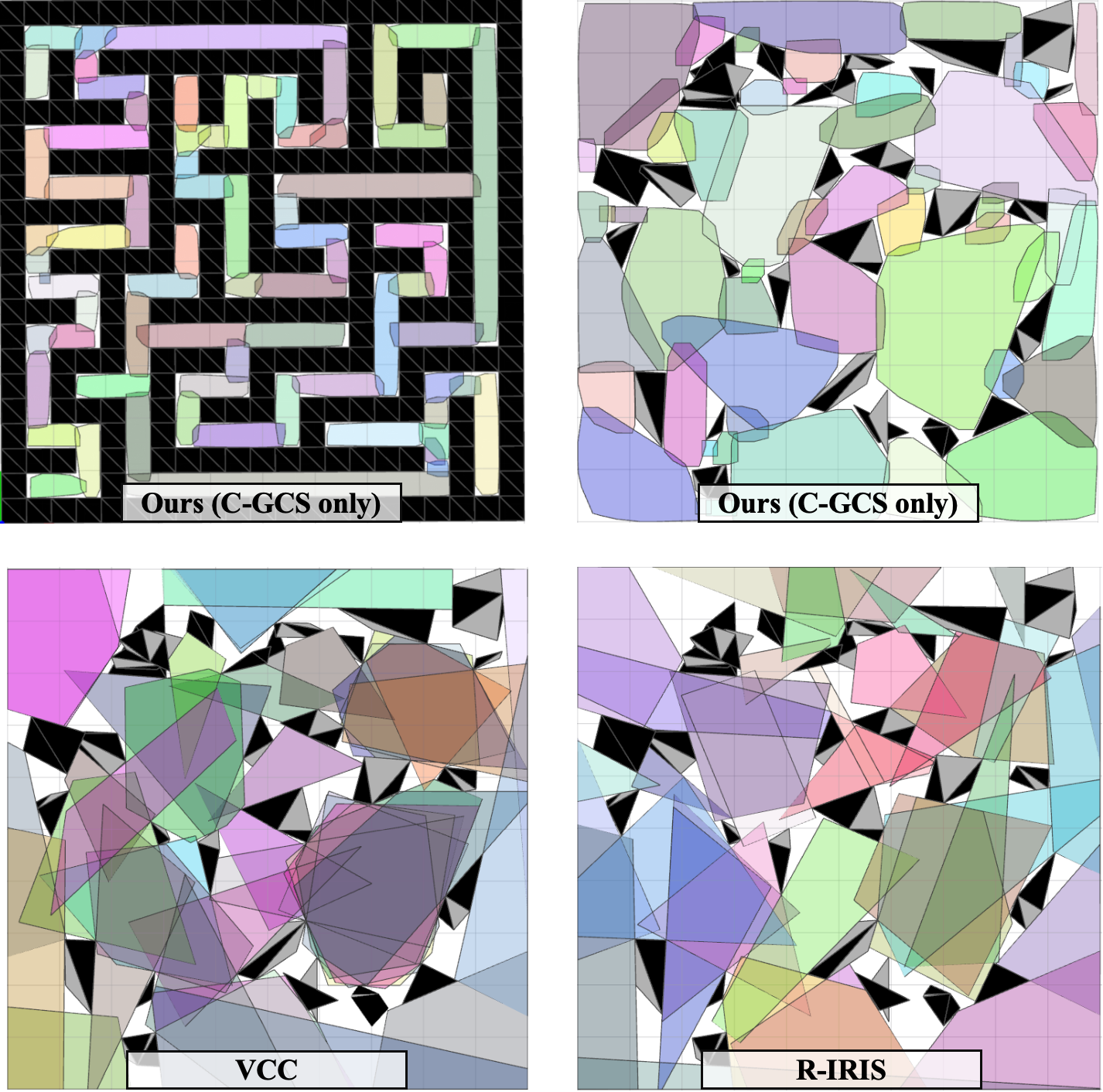}

  \vspace{-0.2cm}

  \caption{
  Visualization of GCS construction in 2-D maze (narrow) and clutter (medium-density). \textbf{Top:} Our method (showing C-GCS only) generates a structured and compact GCS in both environments. \textbf{Bottom:} A comparison with VCC and R-IRIS in cluttered environments shows our method's efficiency in capturing free space topology and reducing redundant overlaps.
  }
  
  \vspace{-0.4cm}

  \label{fig:gcs_visual_comparison}
\end{figure}

We compare our multi-scale GCS (Sec.~\ref{sec:multi-scale-gcs}) with two Drake-based baselines, VCC~\cite{werner2024approximating} and random-seed IRIS (R-IRIS)~\cite{deits2015computing}.
We test nine $10\,\mathrm{m}$-per-side 2-D/3-D environments, varying obstacle density in cluttered scenes and corridor width in mazes (Fig.~\ref{fig:gcs_visual_comparison}).
For fairness, free-space coverage is matched across methods, and all results are averaged over 10 generated maps.
Our method achieves the lowest time in 8 of 9 settings, yielding a $3.8$--$13.4\times$ speedup over VCC in all settings and a $1.3$--$24.8\times$ speedup over R-IRIS in all but one case. 
Average degree and degeneracy are reduced by about $24\%$--$91\%$ and $64\%$--$92\%$ relative to VCC, and by $13\%$--$42\%$ and $21\%$--$49\%$ relative to R-IRIS, indicating a markedly sparser graph for corridor search (see the direct visual comparison in Fig.~\ref{fig:gcs_visual_comparison}).
Overall, our method achieves efficient GCS construction while maintaining a substantially sparser graph for corridor exploration.

\subsubsection{Corridor Generation}

We evaluate corridor generation in ten controlled single-topology 2-D environments with different turn numbers and amplitudes, isolating corridor construction from topological exploration (e.g., Fig.~\ref{fig:corridor_generation_qual}).
We compare reference-path-based IMPC~\cite{chen2023MultiRobot} and SFC~\cite{liu2017planning} with GCS-based R-IRIS~\cite{deits2015computing}, VCC~\cite{werner2024approximating}, and our method.
IMPC and SFC use the same A* reference path with an additional obstacle-inflation margin $\delta_{A^*}\in\{0.05,0.20\}\,\mathrm{m}$, while GCS-based methods extract corridors from the constructed GCS using the same BFS.
Table~\ref{tab:corridor_generation_metrics} shows that reference-path-based baselines are sensitive to $\delta_{A^*}$, which changes the average duration by $+10.0\%$ for IMPC and $-9.7\%$ for SFC.
Without any reference-path tuning, our method remains within $2.0\%$ of the best IMPC duration and is $1.8\%$ and $10.0\%$ shorter than R-IRIS and VCC, respectively.
These results show competitive trajectory-level performance without an A* reference path or an additional reference-path inflation margin.

\begin{table}[t]
\centering
\begin{threeparttable}
\caption{Trajectory-level comparison for corridor generation over ten single-topology 2-D environments. Values are mean $\pm$ std.}
\vspace{-0.2cm}
\label{tab:corridor_generation_metrics}
\footnotesize
\setlength{\tabcolsep}{3.2pt}
\begin{tabular}{l c c c c}
\toprule
\textbf{Method}
& \textbf{$\delta_{A^*}$\textsuperscript{1} (m)}
& \textbf{Duration (s)}
& \textbf{Dur. Ratio\textsuperscript{2}}
& \textbf{Mean Speed (m/s)} \\
\midrule
\multicolumn{5}{c}{\textit{Reference-path-based (A* + corridor generation)}} \\
\midrule
SFC & 0.05 & $31.33 \pm 4.88$ & $1.116 \pm 0.062$ & $0.826 \pm 0.052$ \\
SFC & 0.20 & $28.29 \pm 4.26$ & $1.008 \pm 0.025$ & $0.881 \pm 0.036$ \\
IMPC & 0.05 & $\mathbf{27.49 \pm 4.46}$ & $\mathbf{0.979 \pm 0.033}$ & $\mathbf{0.901 \pm 0.024}$ \\
IMPC & 0.20 & $30.24 \pm 4.40$ & $1.078 \pm 0.020$ & $0.846 \pm 0.037$ \\
\midrule
\multicolumn{5}{c}{\textit{GCS-based (GCS + corridor extraction)}} \\
\midrule
VCC & -- & $31.15 \pm 5.24$ & $1.117 \pm 0.178$ & $0.842 \pm 0.077$ \\
R-IRIS & -- & $28.54 \pm 5.10$ & $1.014 \pm 0.047$ & $0.877 \pm 0.032$ \\
Ours & -- & $28.03 \pm 3.91$ & $1.000 \pm 0.000$ & $0.889 \pm 0.043$ \\
\bottomrule
\end{tabular}

\vspace{0.03cm}
\begin{tablenotes}[flushleft]
\footnotesize
\item[1] $\delta_{A^*}$ denotes the additional obstacle inflation margin used in A* reference-path search; ``--'' indicates not applicable.
\item[2] Duration ratios are computed per environment relative to Ours and averaged.
\end{tablenotes}
\end{threeparttable}
\vspace{-0.35cm}
\end{table}

\begin{figure}[!t]
  \centering
  \includegraphics[width=0.9\linewidth]{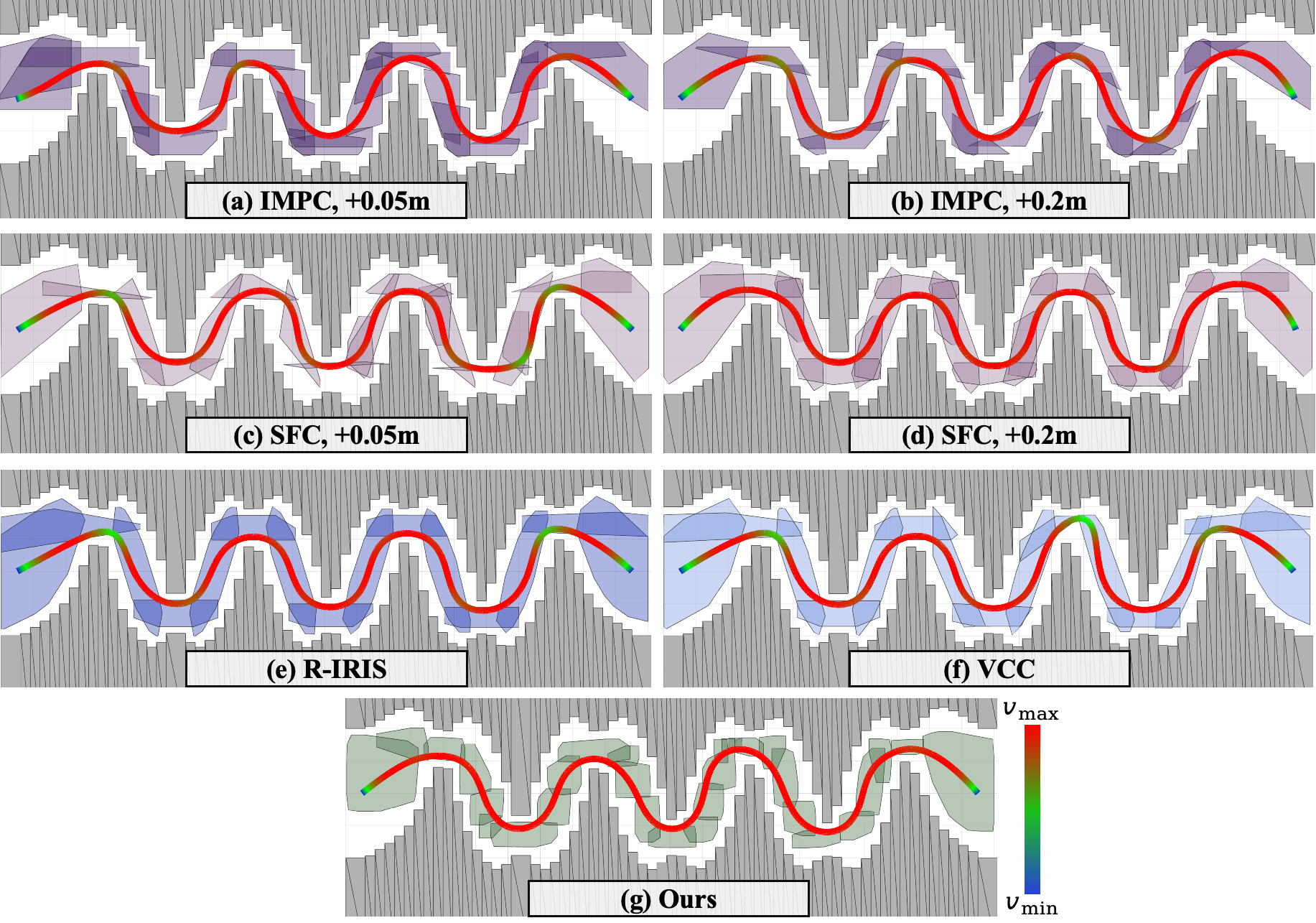}

  \vspace{-0.2cm}

  \caption{
  Corridor generation and resulting trajectories in a single-topology environment.
  Reference-path-based methods (IMPC/SFC) use additional inflation during A* path search to avoid overly boundary-hugging references that can yield degenerate corridors.
  GCS-based methods (R-IRIS/VCC/Ours) extract a corridor via graph search on the constructed GCS without additional inflation tuning.
  }

  \vspace{-0.4cm}

  \label{fig:corridor_generation_qual}
\end{figure}

\subsubsection{Homotopy Awareness} \label{sec:homotopy-awareness}

We further evaluate homotopy-aware corridor exploration over ten 2-D and ten 3-D multi-topology trials based on the environment families illustrated in Fig.~\ref{fig:all_results} (top left, bottom). The 2-D trials use five perturbed obstacle layouts with two diagonal start-goal pairs each, while the 3-D trials vary wall-opening locations with fixed start and goal. We compare SFC~\cite{liu2017planning}, R-IRIS~\cite{deits2015computing}, VCC~\cite{werner2024approximating}, and our method; for GCS-based methods, the top 10 BFS corridors are optimized and the shortest-duration trajectory is selected.
As shown in Table~\ref{tab:homotopy_results}, our method achieves the shortest duration and highest mean speed in both 2-D and 3-D, reducing duration by $13.7$--$19.8\%$ in 2-D and $12.3$--$16.4\%$ in 3-D relative to the baselines. Although SFC yields the shortest geometric length by following a shortest A* path, this does not translate into the fastest dynamically feasible trajectory. These results indicate that our method more reliably identifies topology-distinct corridors that are also dynamically favorable.

\begin{table}[t]
\centering
\caption{Homotopy-aware corridor exploration results in 2-D and 3-D multi-topology environments. Values are mean $\pm$ std.}
\vspace{-0.2cm}
\label{tab:homotopy_results}
\footnotesize
\setlength{\tabcolsep}{2.8pt}
\begin{tabular}{c c c c c c}
\toprule
\textbf{}
& \textbf{Method}
& \textbf{Dur. (s)}
& \textbf{L. (m)}
& \textbf{M.S. (m/s)}
& \textbf{$J_{\mathrm{ang}}$ ($\mathrm{rad}^2$)} \\
\midrule

\multirow{4}{*}{2-D}
& SFC
& $38.7 \pm 9.8$
& $\mathbf{28.0 \pm 1.5}$
& $0.75 \pm 0.13$
& $0.027 \pm 0.023$ \\

& R-IRIS
& $38.0 \pm 3.2$
& $30.2 \pm 1.9$
& $0.80 \pm 0.03$
& $0.021 \pm 0.008$ \\

& VCC
& $40.9 \pm 5.7$
& $30.7 \pm 3.1$
& $0.76 \pm 0.05$
& $\mathbf{0.017 \pm 0.013}$ \\

& Ours
& $\mathbf{32.8 \pm 1.7}$
& $28.5 \pm 1.1$
& $\mathbf{0.87 \pm 0.05}$
& $0.023 \pm 0.013$ \\

\midrule

\multirow{4}{*}{3-D}
& SFC
& $28.1 \pm 5.4$
& $\mathbf{19.3 \pm 0.7}$
& $0.71 \pm 0.11$
& $0.045 \pm 0.040$ \\

& R-IRIS
& $26.9 \pm 3.2$
& $20.4 \pm 1.3$
& $0.77 \pm 0.07$
& $0.026 \pm 0.020$ \\

& VCC
& $26.8 \pm 2.6$
& $20.9 \pm 1.6$
& $0.78 \pm 0.04$
& $0.023 \pm 0.015$ \\

& Ours
& $\mathbf{23.5 \pm 1.0}$
& $19.9 \pm 0.8$
& $\mathbf{0.85 \pm 0.02}$
& $\mathbf{0.018 \pm 0.008}$ \\

\bottomrule
\end{tabular}

\vspace{0.1cm}
\parbox{0.98\columnwidth}{\footnotesize Dur. = Duration, L. = Length, M.S. = Mean Speed.}
\vspace{-0.25cm}
\end{table}

\subsubsection{Local-Update Characterization}

\begin{figure}[!t]
  \centering
  \includegraphics[width=1.0\linewidth]{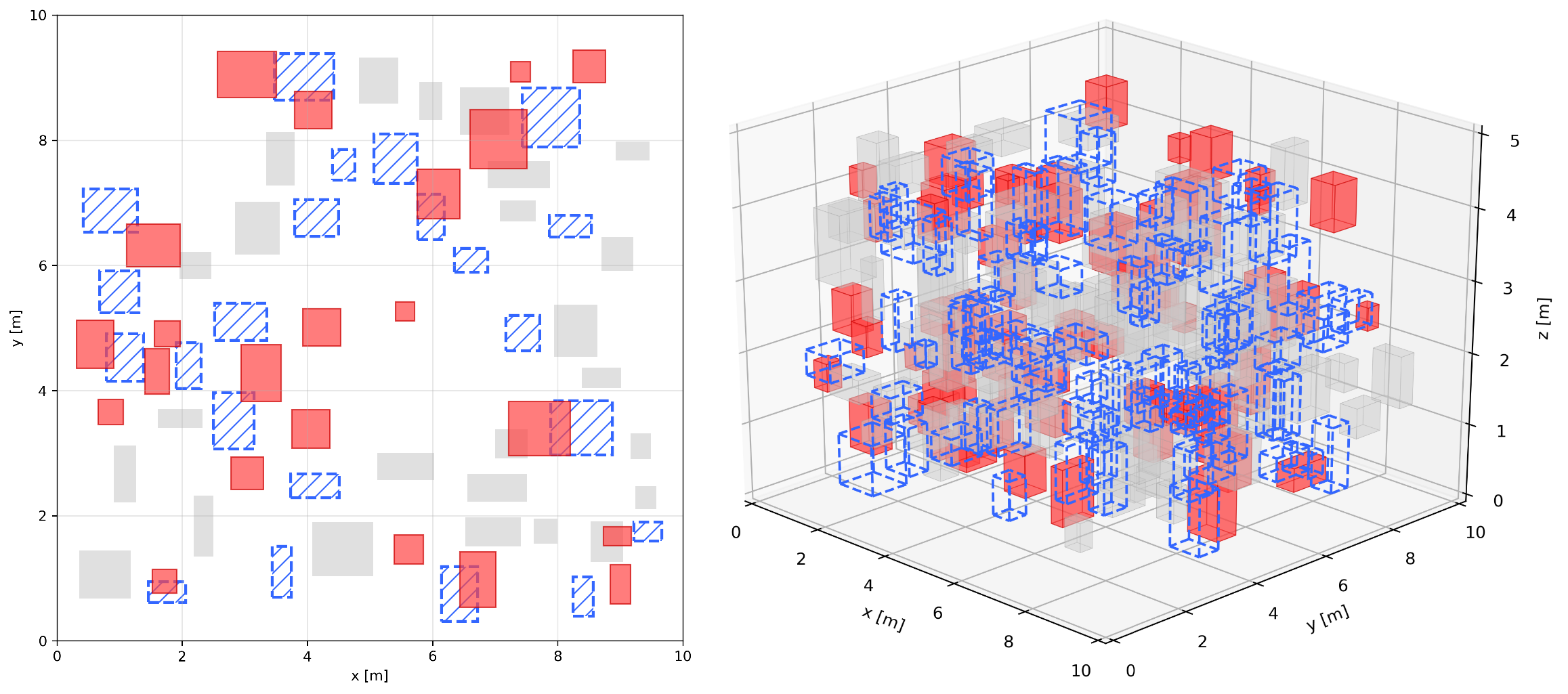}

  \vspace{-0.2cm}

  \caption{ 
  Randomly generated dense, large-mismatch environments used for local-update characterization.
  \textbf{Left:} 2-D case. \textbf{Right:} 3-D case.
  Gray obstacles are shared by the prior and actual maps; red obstacles are newly added to the actual map; blue dashed obstacles exist only in the prior map.
  }

  \vspace{-0.2cm}

  \label{fig:local_update_env}
\end{figure}

\begin{table}[t]
\centering
\begin{threeparttable}
\caption{Local-update characterization. Values are means over five trials.}
\vspace{-0.2cm}
\label{tab:local_update}
\footnotesize
\setlength{\tabcolsep}{3.0pt}
\begin{tabular}{c c c c c}
\toprule
\textbf{Dim.}
& \textbf{Setting}\textsuperscript{1}
& \textbf{$T_{\mathrm{local}}$ (s)}
& \textbf{Speedup}\textsuperscript{2}
& \textbf{Added FCS}\textsuperscript{3} \\
\midrule
\multirow{3}{*}{2-D}
& Small mismatch
& $0.020$--$0.023$
& $19.2$--$24.0\times$
& $73$--$86$ \\
& Medium mismatch
& $0.026$--$0.037$
& $11.5$--$17.5\times$
& $78$--$149$ \\
& Large mismatch
& $0.042$--$0.063$
& $6.4$--$10.1\times$
& $174$--$258$ \\
\midrule
\multirow{4}{*}{3-D}
& Med./small
& $0.287$
& $3.49\times$
& $199$ \\
& Med./medium
& $0.359$
& $4.06\times$
& $219$ \\
& Med./large
& $0.604$
& $1.73\times$
& $395$ \\
& Dense/large
& $0.637$
& $1.72\times$
& $233$ \\
\midrule
2-D
& $r_s:1.0{\rightarrow}3.0\,\mathrm{m}$
& $0.013{\rightarrow}0.050$
& $27.4{\rightarrow}7.1\times$
& $42{\rightarrow}239$ \\
\bottomrule
\end{tabular}

\vspace{0.03cm}
\begin{tablenotes}[flushleft]
\footnotesize
\item[1] The first three 2-D rows report ranges of five-trial means over sparse, medium, and dense maps. For 3-D, settings list density/mismatch. The radius row uses medium density and medium mismatch.
\item[2] Speedup is computed against complete F-GCS/C-GCS reconstruction.
\item[3] FCS denotes newly added fine-scale convex sets.
\end{tablenotes}
\end{threeparttable}
\vspace{-0.2cm}
\end{table}

In Table~\ref{tab:local_update}, we characterize the adaptive local-update regime by varying obstacle density, environment mismatch, and sensing radius.
Because the update region is bounded by the sensing radius, dense or large-mismatch cases (Fig.~\ref{fig:local_update_env}) may reconstruct most of the affected local subgraph, but not the complete global F-GCS/C-GCS unless the sensing region covers the entire workspace.
Across the 2-D density/mismatch sweep, local updates take $0.020$--$0.063\,\mathrm{s}$, while speedup over complete reconstruction decreases from $19.2$--$24.0\times$ for small mismatch to $6.4$--$10.1\times$ for large mismatch.
Increasing the sensing radius from $1.0$ to $3.0\,\mathrm{m}$ raises the added FCS count from $42$ to $239$ and reduces speedup from $27.4\times$ to $7.1\times$.
In 3-D, speedup remains $1.73\times$ for medium-density large mismatch and $1.72\times$ in the dense--large setting.
Thus, local updates are most effective for smaller affected regions but remain faster than complete reconstruction in all tested dense and large-mismatch settings.

\begin{figure}[!t]
  \centering
  \includegraphics[width=0.95\linewidth]{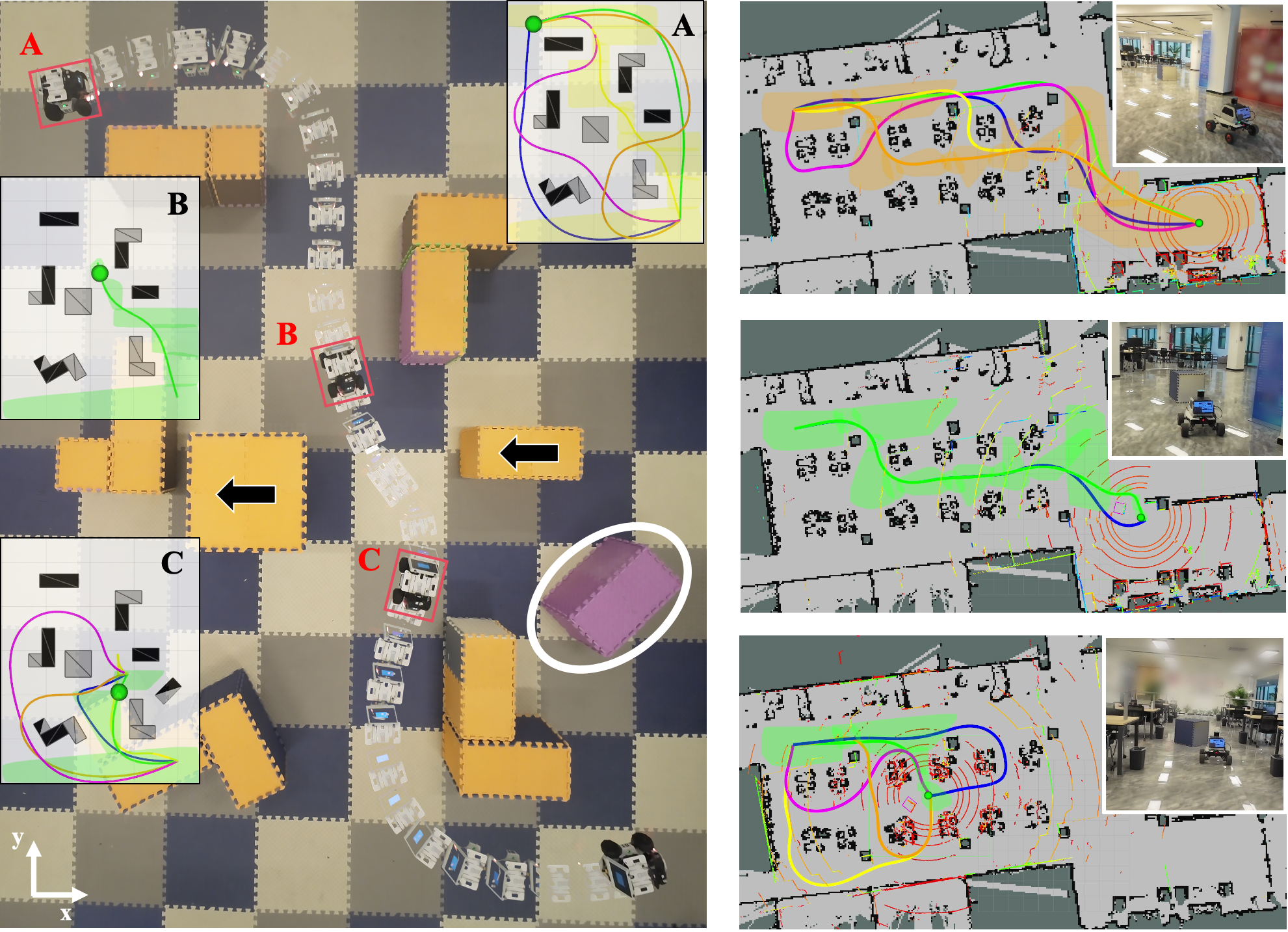}

  \vspace{-0.2cm}

  \caption{
  Snapshots from representative 2-D ground robot experiments.
  \textbf{Left:} Top view of the OptiTrack-based ground robot experiment.
  Black arrows indicate translated obstacles inconsistent with the prior map, while the purple obstacle circled in white was previously unknown.
  \textbf{Right:} Snapshots of the onboard-localization ground robot experiment.
  }

  \vspace{-0.1cm}

  \label{fig:exp_fig_2d}
\end{figure}

\begin{figure}[!t]
  \centering
  \includegraphics[width=0.95\linewidth]{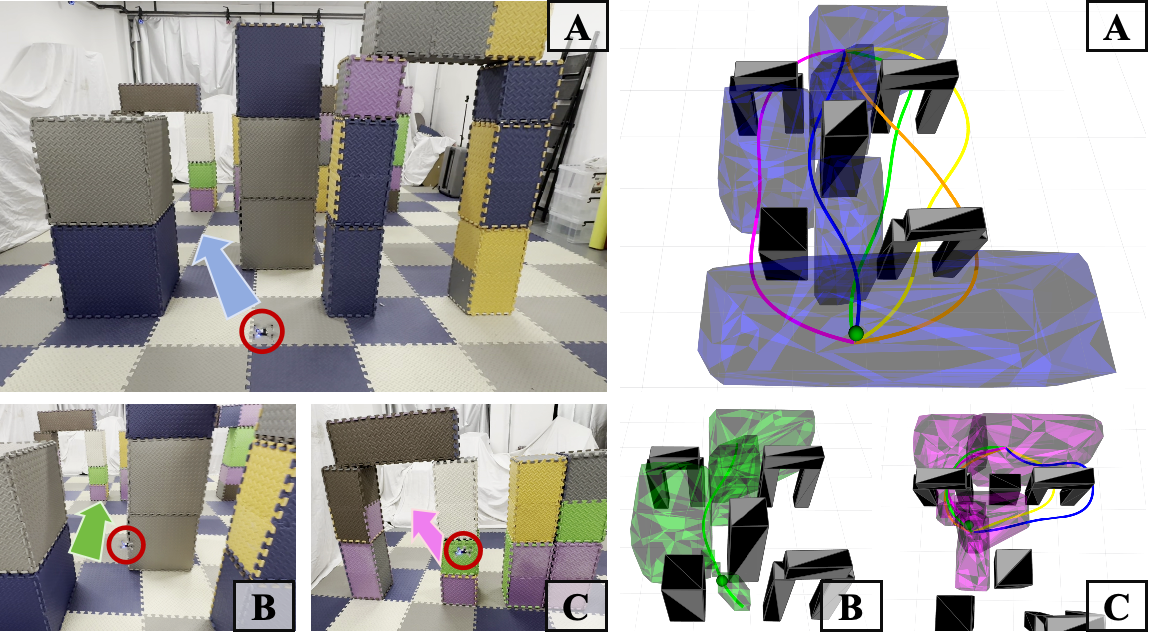}

  \vspace{-0.2cm}

  \caption{
  Snapshots from a representative 3-D quadrotor experiment.
  Panels A--C show initial planning, an online local update, and global replanning, respectively.
  Each stage shows the physical experiment together with the corresponding planning visualization.
  }

  \vspace{-0.2cm}

  \label{fig:exp_fig_3d}
\end{figure}

\subsection{Hardware Experiments}

We validate the complete pipeline in three physical settings: an OptiTrack-based 2-D ground robot, an OptiTrack-based 3-D quadrotor, and an onboard-localization ground robot.
In the first two settings, OptiTrack provides state estimation and planning runs on a laptop, whereas the onboard setting uses AMCL for localization and LiDAR for perception, with the planner running onboard without external motion capture.
The corresponding workspaces measure $4.5\,\mathrm{m}\times5.5\,\mathrm{m}$, $4.5\,\mathrm{m}\times5.5\,\mathrm{m}\times1.5\,\mathrm{m}$, and $33\,\mathrm{m}\times18\,\mathrm{m}$.
Obstacle updates are revealed within the sensing radius, and the top five candidate corridors are enumerated in all experiments.
As summarized in Table~\ref{tab:hardware_summary}, all 15 trials, spanning different start--goal pairs and obstacle-update configurations, reached the goal without collision or timeout under translated and previously unknown obstacles.
In three representative trials shown in Figs.~\ref{fig:exp_fig_2d}--\ref{fig:exp_fig_3d}, offline GCS construction achieves $93.28$--$97.52\%$ coverage in $0.302$--$2.540\,\mathrm{s}$.
Summing the reported F-GCS update, local repair, and global fallback components gives $0.143$--$0.364\,\mathrm{s}$ per update event, with global fallback taking at most $0.179\,\mathrm{s}$ when triggered.

\begin{table}[t]
\centering
\begin{threeparttable}
\caption{Hardware experiment summary.}
\label{tab:hardware_summary}

\begingroup
\scriptsize
\setlength{\tabcolsep}{2.1pt}
\renewcommand{\arraystretch}{1.08}
\begin{tabular}{lcccccccc}
\toprule
\multicolumn{9}{c}{\textbf{Representative timing statistics\textsuperscript{2} from Figs.~\ref{fig:exp_fig_2d}--\ref{fig:exp_fig_3d}}} \\
\midrule
\multirow{2}{*}{\textbf{Setting}\textsuperscript{1}}
& \multicolumn{2}{c}{\textbf{Offline}}
& \multicolumn{3}{c}{\textbf{Update 1}}
& \multicolumn{3}{c}{\textbf{Update 2}} \\
\cmidrule(lr){2-3}\cmidrule(lr){4-6}\cmidrule(lr){7-9}
& \textbf{$T$ (s)} & \textbf{Cov. ($\%$)}
& \textbf{$U_f$ (s)} & \textbf{$R$ (s)} & \textbf{$G$ (s)}
& \textbf{$U_f$ (s)} & \textbf{$R$ (s)} & \textbf{$G$ (s)} \\
\midrule
2-D GR
& 0.302 & 97.52
& 0.124 & 0.019 & --
& 0.127 & 0.008 & 0.059 \\
3-D Quad.
& 2.540 & 96.36
& 0.286 & 0.041 & --
& 0.289 & 0.006 & 0.069 \\
OB-GR
& 1.417 & 93.28
& 0.063 & 0.152 & --
& 0.060 & 0.014 & 0.179 \\
\midrule
\multicolumn{9}{c}{\textbf{Execution statistics: median [min, max] over five trials}} \\
\midrule
\textbf{Setting}\textsuperscript{1}
& \multicolumn{3}{c}{\textbf{Travel time (s)}}
& \multicolumn{3}{c}{\textbf{Length (m)}}
& \multicolumn{2}{c}{\textbf{$J_{\mathrm{ang}}$ ($\mathrm{rad}^2$)}} \\
\midrule
2-D GR
& \multicolumn{3}{c}{$18.33~[14.98,31.15]$}
& \multicolumn{3}{c}{$7.53~[6.95,13.39]$}
& \multicolumn{2}{c}{$1.44~[1.19,2.82]$} \\
3-D Quad.
& \multicolumn{3}{c}{$16.11~[14.26,17.12]$}
& \multicolumn{3}{c}{$6.57~[5.22,6.65]$}
& \multicolumn{2}{c}{$0.18~[0.16,0.27]$} \\
OB-GR
& \multicolumn{3}{c}{$47.80~[24.60,56.80]$}
& \multicolumn{3}{c}{$23.86~[11.89,30.31]$}
& \multicolumn{2}{c}{$0.87~[0.44,3.12]$} \\
\bottomrule
\end{tabular}
\endgroup

\vspace{0.03cm}
\begin{tablenotes}[flushleft]
\footnotesize
\item[1] 2-D GR: the OptiTrack-based ground robot; 3-D Quad.: the OptiTrack-based quadrotor; OB-GR: the onboard-localization ground robot.
\item[2] $T$: offline GCS construction time; Cov.: free-space coverage; $U_f$: F-GCS update time; $R$: local repair time; $G$: global fallback time; ``--'' indicates that global fallback was not triggered.
\end{tablenotes}

\end{threeparttable}
\vspace{-0.35cm}
\end{table}

\section{Conclusion}

This paper presents a reference-path-free, homotopy-aware corridor generation framework using graphs of convex sets.
By representing corridors as convex-set sequences and extending visibility deformation and uniform visibility deformation to them, the method reduces the geometric and topological biases of predefined reference paths while enabling corridor-level similarity reasoning and redundancy reduction.
The adaptive multi-scale graph combines fine-scale geometric fidelity with a compact coarse-scale connectivity abstraction, supporting efficient corridor exploration and localized updates under map uncertainty without full graph reconstruction.
Numerical and hardware experiments validate corridor generation, homotopy-aware exploration, and adaptation to translated and previously unknown obstacles.
Future work will address dynamic environments and multi-agent settings with time-varying and interacting constraints.

\bibliographystyle{IEEEtran}
\bibliography{ref}

\end{document}